\documentclass{article}

\usepackage[utf8]{inputenc} 
\usepackage[T1]{fontenc}    
\usepackage{hyperref}       
\usepackage{url}            
\usepackage{booktabs}       
\usepackage{amsfonts}       
\usepackage{nicefrac}       
\usepackage{microtype}      
\usepackage{xcolor}         
\usepackage{amsmath}
\usepackage{amssymb}
\usepackage{algorithm}
\usepackage{algorithmic}
\usepackage{graphicx}
\usepackage{textcomp}
\usepackage{xcolor}
\usepackage{float}
\usepackage{xspace}
\usepackage{amsthm}
\newtheorem{theorem}{Theorem}
\newtheorem{proposition}{Proposition}
\usepackage{multirow}
\usepackage{subcaption}
\usepackage{authblk}

\newcommand{\vit}{\textrm{ViT}\xspace}
\title{Single GPU Task Adaptation of Pathology Foundation Models for Whole Slide Image Analysis}

\author[1,*]{Neeraj Kumar}
\author[1]{Swaraj Nanda}
\author[1]{Siddharth Singi}
\author[1]{Jamal Benhamida}
\author[1]{David Kim}
\author[1]{Jie-Fu Chen}
\author[1]{Amir Momeni-Boroujeni}
\author[1]{Gregory M. Goldgof}
\author[2,3]{Gabriele Campanella}
\author[1]{Chad Vanderbilt}

\affil[1]{Department of Pathology and Laboratory Medicine, Memorial Sloan Kettering Cancer Center, New York, NY 10065, United States}
\affil[2]{Windreich Department of AI and Human Health, Icahn School of Medicine at Mount Sinai, New York, NY 10029, United States}
\affil[3]{Hasso Platner Institute at Mount Sinai, Icahn School of Medicine at Mount Sinai, New York, NY 10029, United States}
\affil[*]{\texttt{kumarn6@mskcc.org}}

\begin{document}

\maketitle

\begin{abstract}
Pathology foundation models (PFMs) have emerged as powerful tools for analyzing whole slide images (WSIs). However, adapting these pretrained PFMs for specific clinical tasks presents considerable challenges, primarily due to the availability of only weak (WSI-level) labels for gigapixel images, necessitating multiple instance learning (MIL) paradigm for effective WSI analysis. This paper proposes a novel approach for single-GPU \textbf{T}ask \textbf{A}daptation of \textbf{PFM}s (TAPFM) that uses vision transformer (\vit) attention for MIL aggregation while optimizing both for feature representations and attention weights. The proposed approach maintains separate computational graphs for MIL aggregator and the PFM to create stable training dynamics that align with downstream task objectives during end-to-end adaptation. Evaluated on mutation prediction tasks for bladder cancer and lung adenocarcinoma across institutional and TCGA cohorts, TAPFM consistently outperforms conventional approaches, with H-Optimus-0 (TAPFM) outperforming the benchmarks. TAPFM effectively handles multi-label classification of actionable mutations as well. Thus, TAPFM makes adaptation of powerful pre-trained PFMs practical on standard hardware for various clinical applications.
\end{abstract}

\section{Introduction}
Hematoxylin and Eosin (H\&E) staining is the most common slide preparation method in pathology, used for visualizing tissue architecture and cellular details for cancer diagnosis. Whole slide images (WSIs) serve as high-resolution digital representations of these tissue slides, commonly scanned at either $20\times$ or $40\times$ optical magnification that captures $0.50\mu^{2}$ or $0.25\mu^{2}$ of tissue per pixel, respectively. WSIs form the basis of computational pathology that employs machine learning (ML) and computer vision techniques for digital cancer assessment~\cite{song2023artificial}. Due to memory constraints preventing direct processing of gigapixel WSIs and the availability of only slide-level labels for clinical tasks, WSI processing adopts a multiple instance learning (MIL) approach. For MIL, a WSI is represented as a bag of smaller tiles (e.g., $224\times224\times3$) that are processed through neural networks to extract features, then aggregated to generate predictions using only bag-level labels during training~\cite{ilse2018attention,campanella2019clinical} as shown in Figure~\ref{fig:wsi_processing}.


Computational pathology has experienced a paradigm shift with the introduction of pathology foundation models (PFMs), which learn powerful representations from large collections of WSIs through self-supervised pre-training of vision transformers (\vit{}s)~\cite{dosovitskiy2021vit, chen2024uni, vorontsov2024virchow, zimmermann2024virchow2,xu2024gigapath, bioptimus2024}. For specific downstream applications such as gene mutation prediction, survival analysis, and treatment response estimation, existing methods typically use these PFMs as fixed feature extractors and train separate MIL aggregators to generate slide-level predictions~\cite{ilse2018attention, lu2021data}. The fixed-feature approach fails to adapt PFM parameters to the specific downstream task, potentially limiting performance~\cite{mahmood2025benchmarking,alfasly2025validation,campanella2025clinicalbenchmark}. To address this limitation, we propose a novel \textbf{T}ask \textbf{A}daptation of \textbf{P}athology \textbf{F}oundation \textbf{M}odels (TAPFM) approach that: (1) leverages \vit{}'s internal attention mechanism for MIL aggregation, (2) maintains separate computational graphs for PFM and MIL parameter updates with a dual-loss mechanism on a single GPU, and (3) seamlessly integrates with popular PFMs to improve their performance on clinically relevant tasks.

\begin{figure}[ht]
\centering
\includegraphics[width=\textwidth, keepaspectratio]{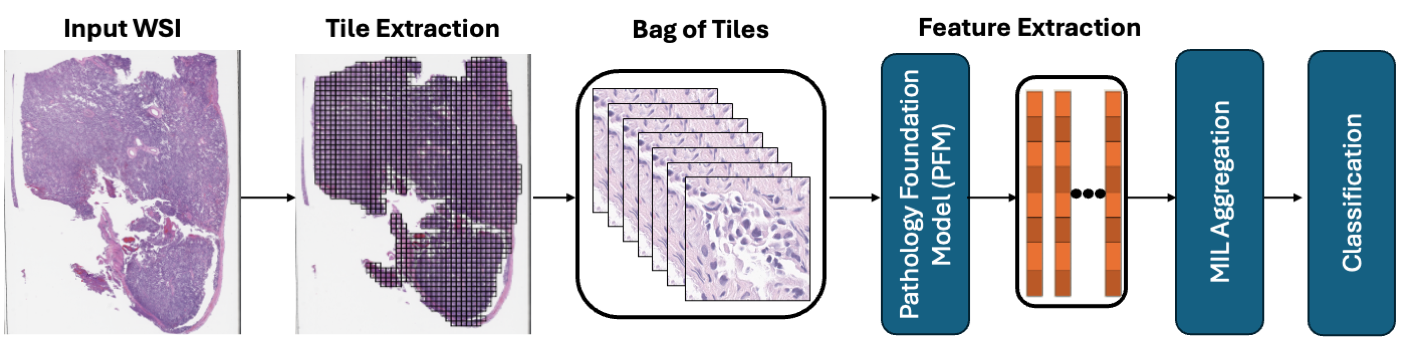}f
\caption{\textbf{WSI Processing Pipeline} -- A representative H\&E stained WSI scanned at $40\times$ optical resolution with $100,000\times125,000$ RGB pixels. For efficient processing, manageable sized tiles (e.g. $224\times224\times3$) shown in black boxes are usually extracted from the tissue region. Bag of tiles obtained from a WSI are then passed through a pathology foundation model to obtain feature representations (one feature vector per tile) which are then combined using multiple instance learning methods to compute a bag feature vector (one per WSI) that is used for downstream tasks such as binary classification (e.g. gene mutation prediction: 0 for absence and 1 for presence of mutation).
}
\label{fig:wsi_processing}
\end{figure}



\section{Related Work}

\subsection{Pathology Foundation Models (PFMs)}


CTransPath \cite{wang2022ctranspath} established an early benchmark by training a hybrid convolutional-transformer architecture on 32,220 WSIs across 25 anatomic sites. HIPT \cite{chen2022hipt} and REMEDIS \cite{azizi2023remedis} explored different architectural approaches with ViT-S (DINO~\cite{caron2021dino}) and ResNet-50 (SimCLR) respectively. Phikon \cite{filiot2023phikon} demonstrated the efficacy of ViT-L trained with iBOT on TCGA slides, while UNI \cite{chen2024uni} significantly expanded scale with its ViT-H architecture trained via DINOv2~\cite{oquab2024dinov2} on 100,000 slides across 20 tissue types. Subsequent models pushed boundaries further with Virchow \cite{vorontsov2024virchow} exploring a ViT-huge model trained on 1.5 million WSIs. This trend toward increased scale continued with Prov-GigaPath \cite{xu2024gigapath} processing 1.3 billion tiles from 171,189 WSIs spanning 31 tissue types and Virchow2 \cite{zimmermann2024virchow2} scaling to 1.7 billion tiles from 3.1 million slides across multiple magnifications. H-optimus-0 \cite{bioptimus2024} leveraged ViT-giant architecture trained on hundreds of millions of tiles from over 500,000 WSIs. Several approaches have explored multimodality, including vision-language models (CONCH \cite{lu2024conch}, PRISM \cite{shaikovski2024prism}, MUSK \cite{xiang2025musk}) and vision-genomics integration (Orpheus~\cite{boehm2025multimodal}), expanding PFMs beyond visual representation learning. Despite this architectural diversity, vision-only models including UNI~\cite{chen2024uni}, GigaPath~\cite{xu2024gigapath}, and H-optimus-0~\cite{bioptimus2024} have demonstrated superior performance on clinically relevant tasks such as cancer diagnosis, mutation prediction, and treatment response estimation~\cite{campanella2025clinicalbenchmark}. The dominant architecture across these PFMs remains \vit{}s~\cite{dosovitskiy2021vit} trained through self-supervised learning, predominantly using DINOv2~\cite{oquab2024dinov2}, on diverse WSI data. \emph{This paper specifically focuses on \vit based PFMs that utilize only pathology images as input.}

\subsection{Multiple Instance Learning (MIL) in Computational Pathology}
MIL methods for WSI analysis have evolved from attention-based mechanisms to spatial-aware architectures. Early MIL approaches used simple aggregation operations such as mean or max pooling to combine tile-level features \cite{campanella2019clinical}. A significant advancement came with attention-based MIL (ABMIL)~\cite{ilse2018attention}, which learns attention weights to selectively focus on diagnostically relevant tiles. CLAM~\cite{lu2021data} extended attention based MIL for multi-class classification. DSMIL \cite{li2021dual} introduced a dual-stream approach coupling max-pooling with attention scoring. VarMIL \cite{schirris2022} incorporated variance modeling to capture tissue heterogeneity while maintaining computational efficiency. Spatially aware MIL methods have also emerged to capture relationships between WSI tiles. TransMIL \cite{shao2021transmil} leveraged transformer architectures with positional encoding, while graph-based approaches like PatchGCN \cite{chen2022hipt} represent tiles as nodes in a graph structure based on physical adjacency. Graph transformer processing (GTP)~\cite{zheng2022graph} further refined this approach by combining graph structures with attention mechanisms. Despite architectural advances, benchmarking studies reveal that performance depends heavily on the specific clinical task and the quality of input embeddings, with no single aggregation method consistently outperforming others across all applications~\cite{chen2024milbenchmarking}.

\subsection{Task Adaptation of PFMs}
MIL methods usually rely on PMFs as fixed feature extractors, creating a disconnect between representation learning and task-specific adaptation for WSI analysis. Li et al.~\cite{li2023task} proposed an Information Bottleneck-based fine-tuning approach that addresses computational constraints through instance sparsification on smaller backbone models (ResNet-50).  The multiple forward passes required by the IB approach make it computationally infeasible for modern large-scale PFMs on single GPU systems due to memory constraints. While recent approaches have attempted to avoid multiple forward passes through end-to-end fine-tuning of large-scale PFMs~\cite{campanella2024finetuning, kumar2024demo}, these methods typically require substantial computational resources spanning tens of GPUs. To the best of our knowledge, no existing approach has leveraged transformer's self attention mechanism for MIL aggregation and task adaptation of large-scale PFMs on a single GPU for downstream clinical applications, while addressing the optimization challenges that arise when jointly training foundation models and MIL aggregators.




\section{Methodology}


\subsection{Problem Formulation}
Let $X = \{x_1, x_2, ..., x_K\}$ represent an H\&E stained WSI composed of $K$ non-overlapping tiles, where each tile $x_i \in \mathbb{R}^{H \times W \times C}$ corresponds to the tissue region extracted from the WSI. For MIL, a cancer patient's WSI has a bag-level label $y \in \{0, 1\}$ for binary classification tasks (e.g., presence or absence of a gene mutation in a patient), but no tile-level annotations. Our goal is to adapt a pretrained PFM $f_\theta$, parametrized by $\theta$, to extract tile-level features that are relevant for a downstream task. A \vit~\cite{dosovitskiy2021vit} based PFM maps each tile to a feature vector $z_i = f_\theta(x_i) \in \mathbb{R}^D$ \cite{chen2024uni, xu2024gigapath, bioptimus2024} by first diving each input tile into a grid of $N$ non-overlapping patches (tokens) of size $P \times P$, with an additional learnable CLS token prepended to the sequence. These tokens are linearly projected and embedded with position information before being processed through multiple self-attention layers to compute tile feature representations.

\subsection{Attention-Based Aggregation}
\label{sec:aggregator}
The CLS token attends to all other tokens, to compute attention weights that indicate the importance of each token for the feature representation. We propose to leverage these attention weights for MIL aggregation. For a WSI with $K$ tiles, let $\mathbf{Z} = [z_1^T, z_2^T, \ldots, z_K^T]^T \in \mathbb{R}^{K \times D}$ denote the feature matrix, where each row $z_i \in \mathbb{R}^D$ is the feature vector (CLS token embedding) for tile $i$. Similarly, let $\mathbf{a} = [a_1, a_2, \ldots, a_K]^T \in \mathbb{R}^K$ denote the vector of attention weights derived from the \vit. For each tile $x_i$, the attention weight $a_i$ is computed as:
\begin{equation}
a_i = \frac{1}{H} \sum_{h=1}^{H} \frac{1}{N} \sum_{j=1}^{N} A^h_{cls,j}
\end{equation}
where $A^h_{cls,j}$ is the attention weight from the CLS token to the $j^{th}$ token in the $h^{th}$ attention head while $H$ and $N$ are the numbers of attention heads and tokens, respectively. The proposed approach maintains a separate computation graph for the aggregator by detaching tile features and attention weights from the PFM's computation graph. The detached attention weights undergo min-max scaling to the range [0,1] followed by softmax normalization, ensuring they form a proper probability distribution for tile importance scoring to compute the bag representation $Z$ as:
\begin{equation}
Z = \mathbf{Z}^T\mathbf{a} = \sum_{i=1}^{K} a_i z_i
\end{equation}
where both $\mathbf{Z}$ and $\mathbf{a}$ are detached from the PFM's computation graph. This ensures that the gradients from the classification loss only flow through the aggregator parameters while keeping the PFM's parameters fixed during this stage of optimization. The bag representation $Z$ is then passed through a linear classifier to predict the bag-level label:
\begin{equation}
\hat{y} = \sigma(W Z + b)
\end{equation}

where $\sigma$ is the sigmoid activation while $W \in \mathbb{R}^{1 \times D}$ and $b \in \mathbb{R}$ are aggregator parameters  $\theta_{agg}: \{W,b\}$, that are learned during backpropagation with weighted cross-entropy loss:
\begin{equation}
\mathcal{L}_{agg}(y, \hat{y}) = -w_y [y \log(\hat{y}) + (1 - y) \log(1 - \hat{y})]
\label{eq:aggloss}
\end{equation}
where $w_y$ is the weight assigned to class $y$ to handle class imbalance.




\subsection{PFM Adaptation}
\label{sec:finetuning}
Instead of employing conventional end-to-end backpropagation, we propose to detach gradients from the aggregator's computation graph to formulate a dedicated loss function for PFM adaptation.


\noindent \textbf {Feature Alignment Loss:} Let us define $\mathbf{G}_z = [g_{z_1}^T, g_{z_2}^T, \ldots, g_{z_K}^T]^T \in \mathbb{R}^{K \times D}$ as the feature gradient matrix where $k^{th}$ row contains the gradient of the aggregator loss (equation~\ref{eq:aggloss}) with respect to the corresponding tile's feature vector. During backpropagation through the aggregator, the gradients with respect to each feature vector are automatically computed (by chain rule) based on their contribution to the bag representation as $g_{z_i} = \frac{\partial \mathcal{L}_{agg}}{\partial z_i} = a_i \frac{\partial \mathcal{L}_{agg}}{\partial Z}$. We propose to detach the feature gradients from the aggregator's computation graph to compute the feature alignment loss as:


\begin{equation}
\mathcal{L}_{feature} = -\text{tr}(\mathbf{Z}\mathbf{G}_z^T) = \sum_{i=1}^{K} \langle z_i, g{z_i} \rangle = \sum_{i=1}^{K} \sum_{d=1}^{D} z_{i,d} \times g_{z_i,d}
\label{eq:featureloss}
\end{equation}

This loss guides feature vectors to move in the direction that reduces the classification loss and can be interpreted as a first-order approximation of the effect of feature changes on the aggregator loss.

\noindent \textbf{Attention Loss:} Let us define the gradient of the aggregator loss with respect to attention weights as $\mathbf{g}_a = [g_{a_1}, g_{a_2}, \ldots, g_{a_K}]^T \in \mathbb{R}^K$. Similar to the feature gradients, the gradient of the aggregator loss with respect to each attention weight is automatically computed by the chain rule as $g_{a_i} = \frac{\partial \mathcal{L}_{agg}}{\partial a_i} = \langle z_i, \frac{\partial \mathcal{L}_{agg}}{\partial Z} \rangle$. We propose to compute the attention loss using the detached attention gradient as:

\begin{equation}
\mathcal{L}_{attention} = \mathbf{a}^T\mathbf{g}_a = \sum_{i=1}^{K} a_i \cdot g_{a_i}
\label{eq:attentionloss}
\end{equation}

This loss encourages attention weights to adjust based on the informativeness of each tile for the downstream task, increasing (or decreasing) weights for informative (or uninformative) tiles.

    \noindent \textbf{Task Adaptation Loss (TAL):} For PFM updates, TAL combines the feature and the attention loss:
\begin{equation}
\mathcal{L}_{PFM} = \mathcal{L}_{feature} + \lambda \mathcal{L}_{attention} 
\label{eq:pfmloss}
\end{equation}

where $\lambda$ is the hyperparameter that controls the relative importance of the attention loss. The PFM parameters, $\theta_{PFM}$, are then updated with the backpropagation using its own loss: $\mathcal{L}_{PFM}$. The training procedure for the proposed approach is presented in Algorithm \ref{alg:TAPFM} with illustraion shown in Figure~\ref{fig:TAPFM_illustration} and its implementation is available at \href{https://github.com/pfmadaptation/tapfm}{https://github.com/pfmadaptation/tapfm}.

\begin{algorithm}
\caption{\textbf{T}ask \textbf{A}daptation of \textbf{P}athology \textbf{F}oundation \textbf{M}odels (TAPFM)}
\label{alg:TAPFM}
\begin{algorithmic}[1]
\STATE \textbf{Input:} Dataset $\mathcal{D}$ of WSIs and labels $\{(X_b, y_b)\}_{b=1}^{M}$, pretrained pathology foundation model $f_{\theta_{PFM}}$, MIL aggregator $f_{\theta_{agg}}$, learning rates ($\eta_{agg}$ and $\eta_{PFM}$)

\FOR{epoch = 1 to num\_epochs}
    \FOR{each $(X_b, y_b)$ in $\mathcal{D}$}
        \STATE Extract K tiles $\{x_1, x_2, ..., x_K\}$ from $X_b$ as tensor $\mathbf{X} \in \mathbb{R}^{H\times W \times C\times K}$ 
        \STATE \textcolor{blue}{// PFM forward pass}
        \STATE $(\mathbf{Z}, \mathbf{A}) = f_{\theta}(\mathbf{X})$ \COMMENT{Extract last layer CLS feature and attention matrices}
        \STATE Compute $\mathbf{a} = [a_1, a_2, \ldots, a_K]^T \in \mathbb{R}^K$ where, $a_i = \frac{1}{H} \sum_{h=1}^{H} \frac{1}{N} \sum_{j=1}^{N} A^h_{cls,j}$ \COMMENT{Average CLS attention across heads and tokens}
%
%
        \STATE \textcolor{blue}{// Detach features and attention from PFM computation graph}
        \STATE $\mathbf{Z}_{detached} = \text{detach}(\mathbf{Z})$ \COMMENT{Detach feature matrix}
        \STATE $\mathbf{a}_{detached} = \text{detach}(\mathbf{a})$ \COMMENT{Detach attention vector}
        \STATE \textcolor{blue}{// Aggregator parameter update}
        \STATE $\mathbf{a}_{detached} = \text{softmax}(\text{minmax}(\mathbf{a}_{detached}))$ \COMMENT{Scale to [0,1] range and normalize}
        \STATE $Z = \mathbf{Z}^T_{detached}\mathbf{a}_{detached}$ \COMMENT{Compute Bag representation}

        \STATE Compute $\hat{y} = f_{\theta_{agg}}(z)$ and $\mathcal{L}_{agg}$ (equation~\ref{eq:aggloss})
        \STATE Backpropagate: $\theta_{agg} \leftarrow \theta_{agg} - \eta_{agg} \nabla_{\theta} \mathcal{L}_{agg}$   
    
        \STATE \textcolor{blue}{// Detach gradients from aggregator computation graph}
        \STATE $\mathbf{G}_z^{detached} = \text{detach}(\mathbf{G}_z)$ \COMMENT{Detach feature gradient matrix}
        \STATE $\mathbf{g}_a^{detached} = \text{detach}(\mathbf{g}_a)$ \COMMENT{Detach attention gradient vector}
        
        \STATE \textcolor{blue}{// PFM fine-tuning with detached gradients}
        \STATE Compute fine-tuning loss $\mathcal{L}_{PFM}$ (Section~\ref{sec:finetuning})

            \STATE Backpropagate: $\theta_{PFM} \leftarrow \theta_{PFM} - \eta_{PFM} \nabla_{\theta_{PFM}} \mathcal{L}_{PFM}$
    \ENDFOR
\ENDFOR
\RETURN Fine-tuned PFM ($f_{\theta_{PFM}}$) and trained MIL aggregator ($f_{\theta_{agg}}$)
\end{algorithmic}
\end{algorithm}

\begin{figure}[ht]
\centering
\includegraphics[width=\textwidth, keepaspectratio]{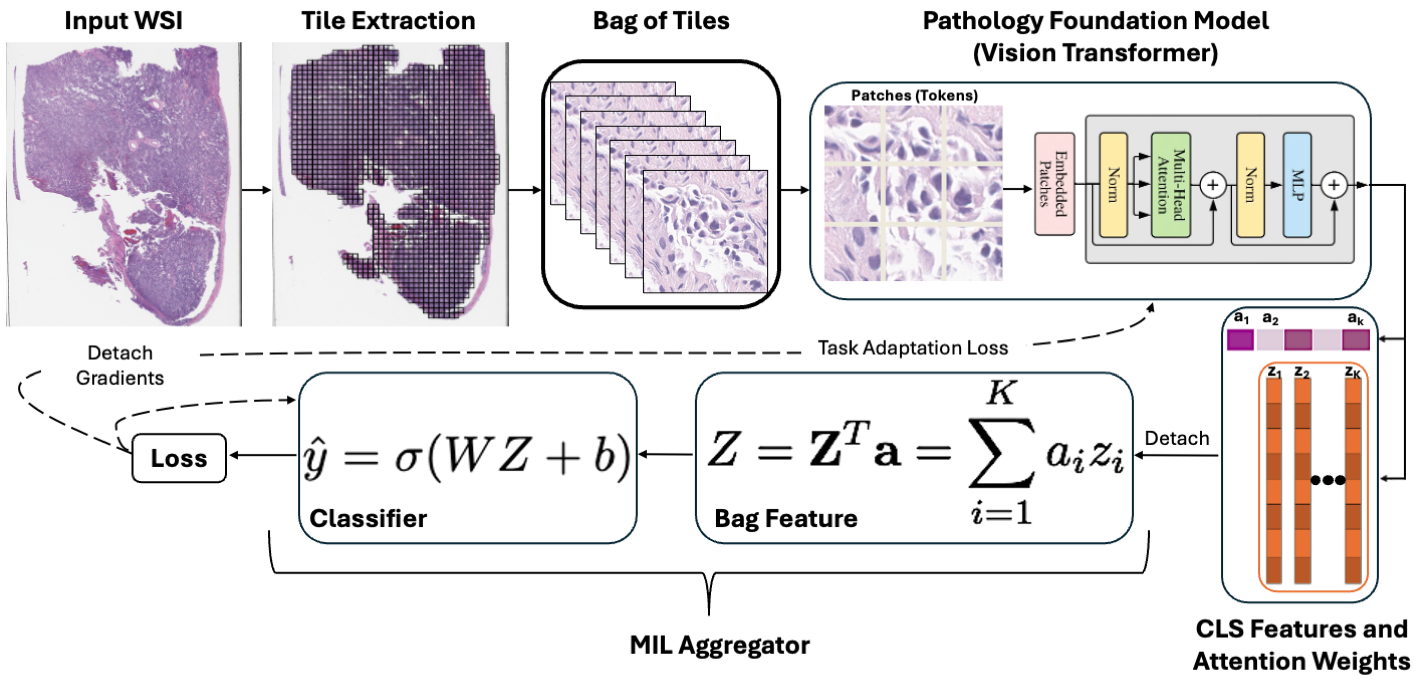}
\caption{\textbf{TAPFM Algorithm Overview} -- The proposed approach processes WSI tiles through a pathology foundation model (PFM) to extract features $\mathbf{Z}$ and attention weights $\mathbf{a}$ (solid arrows). These are detached from the PFM's computation graph and used by the MIL aggregator to compute bag-level predictions and aggregator loss $\mathcal{L}_{agg}$ (solid arrows). During backpropagation (dashed arrows), gradients from $\mathcal{L}_{agg}$ are detached to formulate the task adaptation loss $\mathcal{L}_{PFM}$ (equation~\ref{eq:pfmloss}) for fine-tuning PFM parameters, while aggregator parameters are updated using $\mathcal{L}_{agg}$. This dual-optimization approach with separate computational graphs enables stable task adaptation on a single GPU.}

\label{fig:TAPFM_illustration}
\end{figure}



\subsection{Theoretical Analysis}



The key innovation in our approach lies in the decoupling of the optimization process. In conventional end-to-end training with a unified computational graph, gradients flow through both the PFM and aggregator simultaneously:
{\small
\begin{equation}
\nabla_{\theta_{agg},\theta_{PFM}}\mathcal{L} = \left(\frac{\partial \mathcal{L}_{agg}}{\partial \hat{y}} \cdot \frac{\partial \hat{y}}{\partial \theta_{agg}},\frac{\partial \mathcal{L}_{agg}}{\partial \hat{y}} \cdot \frac{\partial \hat{y}}{\partial Z} \cdot \left(\frac{\partial Z}{\partial \mathbf{Z}} \cdot \frac{\partial \mathbf{Z}}{\partial \theta_{PFM}} + \frac{\partial Z}{\partial \mathbf{a}} \cdot \frac{\partial \mathbf{a}}{\partial \theta_{PFM}}\right) \right)
\end{equation}
}


TAPFM instead implements the following two-stage optimization:

{\small
\begin{equation}
\nabla_{\theta_{agg}}\mathcal{L}_{agg} = \frac{\partial \mathcal{L}_{agg}}{\partial \hat{y}} \cdot \frac{\partial \hat{y}}{\partial \theta_{agg}} \quad \textrm{(Stage 1: Aggregator Update)}
\end{equation}
}
{\small
\begin{align}
\scriptstyle
\nabla_{\theta_{PFM}}\mathcal{L}_{PFM} 
&= \text{detach}\left(\frac{\partial \mathcal{L}_{agg}}{\partial Z}\right) \cdot \Bigg(
    \frac{\partial Z}{\partial \mathbf{a}} \cdot \frac{\partial \mathbf{a}}{\partial \theta_{PFM}} + \frac{\partial Z}{\partial \mathbf{Z}} \cdot \frac{\partial \mathbf{Z}}{\partial \theta_{PFM}} 
\Bigg) \textrm{(Stage 2: PFM Update)}
\end{align}
}

This design resolves the circular dependency challenge in jointly optimizing PFM and MIL aggregator parameters by detaching gradients between optimization stages, that would otherwise create unstable training dynamics.

\begin{proposition}[Gradient Stabilization]

The TAPFM approach breaks the circular dependency at each iteration $t$ by enforcing:
{\small
\begin{equation}
\frac{\partial \mathcal{L}_{agg}}{\partial \theta_{agg}}|_{t} \propto g(\theta_{PFM_{t-1}}, \theta_{agg_{t-1}}) \quad \text{and} \quad \frac{\partial \mathcal{L}_{PFM}}{\partial \theta_{PFM}}|_{t} \propto f(\theta_{PFM_{t-1}}, \theta_{agg_{t}})
\end{equation}
}
resulting in more stable parameter trajectories than joint optimization.
\end{proposition}

\begin{proof}
In joint optimization, the parameter updates create an implicit feedback loop:

{\small
\begin{align}
\theta_{PFM_{t}} &= \theta_{PFM_{t-1}} - \eta_{PFM} \nabla_{\theta_{PFM}}\mathcal{L}_{agg}(\theta_{PFM_{t-1}}, \theta_{agg_{t-1}})\\
\theta_{agg_{t}} &= \theta_{agg_{t-1}} - \eta_{agg} \nabla_{\theta_{agg}}\mathcal{L}_{agg}(\theta_{PFM_{t}}, \theta_{agg_{t-1}})
\end{align}
}

Note that $\theta_{agg_{t}}$ depends on $\theta_{PFM_{t}}$, which itself depends on $\theta_{agg_{t-1}}$. This creates a circular dependency where each parameter set is chasing a moving target. However, the proposed approach breaks this loop by detaching the gradient computation graphs:

{\small
\begin{align}
\theta_{agg_{t}} &= \theta_{agg_{t-1}} - \eta_{agg} \nabla_{\theta_{agg}}\mathcal{L}_{agg}(\text{detach}(\theta_{PFM_{t-1}}), \theta_{agg_{t-1}})\\
\theta_{PFM_{t}} &= \theta_{PFM_{t-1}} - \eta_{PFM} \nabla_{\theta_{PFM}}\mathcal{L}_{PFM}(\theta_{PFM_{t-1}}, \text{detach}(\theta_{agg_{t}}))
\end{align}
}

This detaching operation ensures that during aggregator optimization, $\theta_{PFM}$ is treated as a constant, and during PFM optimization, the updated $\theta_{agg}$ influences the loss but does not receive gradient updates. This effectively eliminates the circular dependency and stabilizes training.
\end{proof}

Additional properties of the proposed Algorithm~\ref{alg:TAPFM} are elaborated in the appendix: permutation invariance (Appendix~\ref{sec:perm_invariance}), computational and space complexity analysis(Appendix~\ref{app:comp_complexity} and \ref{app:spacecomplexity}), and prevention
of catastrophic forgetting during task adaptation (Appendix~\ref{app:catastrophic_forgetting}).

\subsection{Multi-label Classification}
\label{sec:mlclassification}
The proposed Algorith~\ref{alg:TAPFM} can be easily adapted for multi-label classification problems where each patient can be associated with multiple binary labels. Let $y = [y^1, y^2, ..., y^C]$ represent the ground truth vector where $y^j \in \{0,1\}$ indicates the presence or absence of the $j$-th mutation. The aggregator loss for each class is computed using the weighted cross-entropy loss as in Equation~\ref{eq:aggloss}. Note that the each output node in MIL aggregator is considered as class-specific binary predictor for multi-label classification. The total aggregator loss is then formulated as a weighted sum $\mathcal{L}_{agg} = \sum_{j=1}^{C} \alpha_j \mathcal{L}_{agg}^j$, where $\alpha_j$ is the weight assigned to class $j$, calculated as the inverse of the ratio of positive examples of that class to the total number of training instances. 



\section{Experiments}

The proposed TAPM approach is evaluated on clinically relevant mutation prediction tasks using institutional and public cohorts of bladder cancer (BLCA) and lung adenocarcinoma (LUAD) patients.

\textbf{Datasets}: The institutional dataset consists of H\&E WSIs of $2,030$ BLCA and $8,820$ LUAD patients collected during routine clinical care at *institute*, including both $20\times$ and $40\times$ ($\sim30$\% of all WSIs in each cohort) magnifications to reflect real-world data acquisition variability. WSIs from The Cancer Genome Atlas (TCGA) cohorts -- TCGA-BLCA (260 patients) and TCGA-LUAD (438 patients), all scanned at $40\times$ resolution, are exclusively used for external validation to assess generalizability of the proposed approach. Only one WSI per patient is used in all training and validation cohorts. 

For the binary classification task, TAPM is evaluated on two clinically relevant mutation prediction tasks: FGFR3 in BLCA and EGFR in LUAD. For multi-label classification, TAPM's ability to simultaneously predict four actionable mutations in LUAD patients is assessed: EGFR, KRAS, MET, and ALK. The prevalence rate of these mutations across institutional (TCGA) cohorts are: 16\% (14\%) for FGFR3 in BLCA, and in LUAD: 26\% (14\%) for EGFR, 27\% (35\%) for KRAS, 4\% (2\%) for MET, and 3\% (1\%) for ALK.



\textbf{Benchmarks}: State-of-the-art PFMs such as \textbf{UNI} \cite{chen2024uni}, \textbf{GigaPath} \cite{xu2024gigapath}, and \textbf{H-Optimus-0} \cite{bioptimus2024}. Motivated by prior findings that lightweight MIL models can attain clinical performance comparable to computationally expensive aggregators~\cite{chen2024milbenchmarking}, this work employs memory-efficient MIL methods on single-GPU systems: \textbf{ABMIL}~\cite{ilse2018attention}, \textbf{DSMIL}~\cite{li2021dual}, \textbf{CLAM}~\cite{lu2021data}, and \textbf{VarMIL}~\cite{schirris2022}. 

\textbf{Implementation Details}: All reported experiments in this paper are conducted using PyTorch 2.5.1 on a single NVIDIA H100 (80GB memory). Each WSI is processed at native resolution ($20\times$ or $40\times$) to extract non-overlapping 
tiles in a sliding window manner after filtering out the non-tissue regions with Otsu thresholding. For $20\times$ WSIs, tiles of size $224\times 224 \times 3$ pixels are extracted directly, while for $40\times$ images, tiles of size $448 \times 448 \times 3$ pixels are extracted and resized to $224 \times 224 \times 3$ pixels to maintain consistent spatial context. As detailed in the space complexity analysis (Appendix~\ref{app:spacecomplexity}), the memory requirements of the proposed TAPFM method scale quadratically with the number of tokens (patches) processed by \vit{}s. The $224\times 224$ tile size is selected as the optimal dimension that prevents out-of-memory errors while maximizing the contextual information captured per tile. During training, 300, 100, and 75 tiles per WSI per epoch are randomly sample without replacement for UNI, Gigapath, and H-Optimus-0 respectively -- the maximum number of tiles processable for each PFM on a single H100 GPU while maintaining end-to-end fine-tuning capability. \emph{At inference, all tiles obtained from a given WSI are used for downstream mutation prediction tasks.}

For the proposed TAPFM method $\lambda = 1.0$ is used for all experiments. The institutional data was stratified into training (80\%), validation (10\%), and test (10\%) sets maintaining patient level separation with balanced distribution of labels and resolutions across splits. Area Under the ROC Curve (AUC) is used as the primary evaluation metric. Model selection is performed using the validation set, with standard AUC for binary classification tasks and macro-average AUC for multi-label classification tasks determining the best checkpoint for further assessment on the testing sets.


For training, AdamW~\cite{loshchilov2017decoupled} with weight decay of 1e-4 is used as opitimzer, applying differential learning rates of 1e-6 and 1e-5 for PFM and aggregator parameters, respectively. Training data augmentation included random horizontal flips, random rotations (of 90°, 180°, or 270°), and Gaussian blur. A cosine annealing scheduler with warm restarts ($T_{0}=10, T_{mult}=2$) is also used for better convergence~\cite{loshchilov2016sgdr}. Each batch contained 1 WSI, and TAPFM is trained for 20 epochs while all other benchmarks are trained for 50 epochs, with the institutional validation set used to select the best-performing model for all evaluations on the insititutional and TCG testing sets.

\subsection{Results}
\label{sec:results}
Table~\ref{tab:pfm_mil_comparison} shows the performance comparison of three sets of models on the testing data: (1) fixed-PFM with trained MIL aggregators, (2) fine-tuned PFM and MIL aggregators (equivalent to setting $\lambda = 0$ in equation~\ref{eq:pfmloss} with external MIL methods), and (3) proposed TAPFM. It is evident that the proposed TAPFM approach outperforms the other benchmarks across both binary mutation prediction tasks. H-Optimus-0 (TAPFM) consistently achieves the best performance across both institutional and external TCGA testing cohorts, followed by Gigapath (TAPFM), indicates generalizability of the proposed approach.  Table~\ref{tab:multi_label_classification} extends TAPFM evaluation to the more challenging task of simultaneously predicting four actionable mutations in LUAD. H-Optimus-0 (TAPFM) consistently outperforms GigaPath (TAPFM) across all mutations, even for the rate MET and ALK mutations.


\subsection{Runtime Performance}
Training times for BLCA are 12 hours for UNI, 21 hours for Gigapath, and 24 hours for H-Optimus-0. For LUAD cases, training required 2 days 4 hours for UNI, 4 days 2 hours for Gigapath, and 4 days 6 hours for Hoptimus. Inference times per WSI are 4.85 minutes for UNI, 6.38 minutes for Gigapath, and 7.15 minutes for H-Optimus-0. These results confirm that TAPFM enables efficient PFM task adaptation on standard hardware making it suitable for clinical implementation.

\begin{table}[!htb]
\caption{Performance comparison of different PFM and MIL aggregation methods for binary classification tasks in terms of AUC. $N$ indicates the number of patients in the testing cohort. Best and second best models are in bold and underlined text, respectively.}
\label{tab:pfm_mil_comparison}
\centering
\resizebox{\textwidth}{!}{%
\begin{tabular}{lcccc}
\toprule
\multirow{2}{*}{\textbf{Model}} & \multicolumn{2}{c}{\textbf{BLCA FGFR3}} & \multicolumn{2}{c}{\textbf{LUAD EGFR}} \\

\cmidrule(lr){2-3} \cmidrule(lr){4-5}
& \textbf{Institutional Cohort} & \textbf{TCGA} & \textbf{Institutional Cohort} & \textbf{TCGA} \\
& \textbf{$(N = 194)$} & $(N = 260)$ & \textbf{$(N = 876)$} & $(N = 438)$ \\
\midrule
\multicolumn{5}{l}{\textit{Fixed-PFM with Trained MIL Aggregators}} \\
\midrule

UNI + DSMIL & 0.7876 & 0.7928 & 0.7352 & 0.7624 \\
UNI + CLAM & 0.7893 & 0.7912 & 0.7389 & 0.7645 \\
UNI + VarMIL & 0.7912 & 0.7948 & 0.7414 & 0.7663 \\
UNI + ABMIL & 0.7904 & 0.7962 & 0.7396 & 0.7718 \\
\midrule

GigaPath + DSMIL & 0.8232 & 0.8645 & 0.7657 & 0.8153 \\
GigaPath + CLAM & 0.8246 & 0.8673 & 0.7672 & 0.8164 \\
GigaPath + VarMIL & 0.8274 & 0.8689 & 0.7691 & 0.8192 \\
GigaPath + ABMIL & 0.8294 & 0.8712 & 0.7711 & 0.8205 \\
\midrule


H-Optimus-0 + DSMIL & 0.8365 & 0.8731 & 0.7694 & 0.8237 \\
H-Optimus-0 + CLAM & 0.8373 & 0.8752 & 0.7713 & 0.8253 \\
H-Optimus-0 + VarMIL & 0.8401 & 0.8774 & 0.7735 & 0.8282 \\
H-Optimus-0 + ABMIL & 0.8412 & 0.8786 & 0.7742 & 0.8295 \\
\midrule
\multicolumn{5}{l}{\textit{Fine-tuned (FT) PFM with MIL Aggregators}} \\
\midrule

UNI + DSMIL (FT) & 0.8132 & 0.8209 & 0.7526 & 0.7837 \\
UNI + CLAM (FT) & 0.8147 & 0.8223 & 0.7542 & 0.7865 \\
UNI + VarMIL (FT) & 0.8176 & 0.8252 & 0.7568 & 0.7894 \\
UNI + ABMIL (FT) & 0.8193 & 0.8237 & 0.7581 & 0.7922 \\
\midrule
GigaPath + DSMIL (FT) & 0.8372 & 0.8831 & 0.7986 & 0.8351 \\
GigaPath + CLAM (FT) & 0.8381 & 0.8847 & 0.8012 & 0.8377 \\
GigaPath + VarMIL (FT) & 0.8407 & 0.8873 & 0.8043 & 0.8408 \\
GigaPath + ABMIL (FT)  & 0.8393 & 0.8858 & 0.8074 & 0.8421 \\
\midrule

H-Optimus-0 + DSMIL (FT) & 0.8478 & 0.8843 & 0.8121 & 0.8479 \\
H-Optimus-0 + CLAM (FT) & 0.8491 & 0.8859 & 0.8143 & 0.8512 \\
H-Optimus-0 + VarMIL (FT) & 0.8526 & 0.8889 & 0.8167 & 0.8543 \\
H-Optimus-0 + ABMIL (FT) & 0.8512 & 0.8874 & 0.8189 & 0.8529 \\
\midrule
\multicolumn{5}{l}{\textit{Proposed TAPFM Method}} \\
\midrule
UNI (TAPFM) & 0.8415 & 0.8536 & 0.8175 & 0.8309 \\
Gigapath (TAPFM) & \underline{0.8564} & \underline{0.8994} & \underline{0.8426} & \underline{0.8528} \\
H-Optimus-0  (TAPFM) & \textbf{0.8647} & \textbf{0.9021} & \textbf{0.8491} & \textbf{0.8553} \\
\bottomrule
\end{tabular}%
}
\end{table}

\begin{table}[!htb]
\caption{Performance comparison on multi-label classification of actionable mutations in LUAD. Results are reported as AUC. Based on binary classification results (Table~\ref{tab:pfm_mil_comparison}), only the top performers are shown, as performance varied primarily by foundation model type.}
\label{tab:multi_label_classification}
\resizebox{\textwidth}{!}{%
\begin{tabular}{lccccc}
\toprule
\multirow{2}{*}{\textbf{Model}} & \multicolumn{4}{c}{\textbf{Institutional Cohort $(N = 876)$}} & \multirow{2}{*}{\textbf{Macro Average}} \\
\cmidrule(lr){2-5}
& \textbf{EGFR} & \textbf{KRAS} & \textbf{MET} & \textbf{ALK} & \\
\midrule
Gigapath (TAPFM) & 0.8595 & 0.8075 & 0.8350 & 0.8632 & 0.8439 \\
H-Optimus-0 (TAPFM) & 0.8665 & 0.8153 & 0.8420 & 0.8702 & \textbf{0.8510} \\

\midrule
\multicolumn{6}{c}{\textbf{TCGA $(N = 438)$}} \\
\midrule
GigaPath (TAPFM) & 0.8603 & 0.8147 & 0.8179 & 0.8540 & 0.8392  \\

H-Optimus-0 (TAPFM) & 0.8629 & 0.8226 & 0.8241 & 0.8563 & \textbf{0.8439} \\
\bottomrule
\end{tabular}%
}
\end{table}

\subsection{Convergence} 
For binary FGFR3 classification in BLCA, Figure~\ref{fig:fgfr3_blca_val} shows UNI reaching maximum validation performance by epoch 6 (AUC 0.8542), Gigapath by epoch 5 (AUC 0.8764), and H-Optimus-0 by epoch 7 (AUC 0.8960). Additionally, experiments on LUAD datasets (not shown) demonstrated that all PFMs converge within 4 epochs for binary classification tasks. For multi-label LUAD classification, convergence occurred at epoch 10 for Gigapath and epoch 11 for H-Optimus-0.


\begin{figure}[!htb]
    \centering
    \begin{subfigure}[b]{0.3\textwidth}
        \centering
        \includegraphics[width=\textwidth, keepaspectratio]{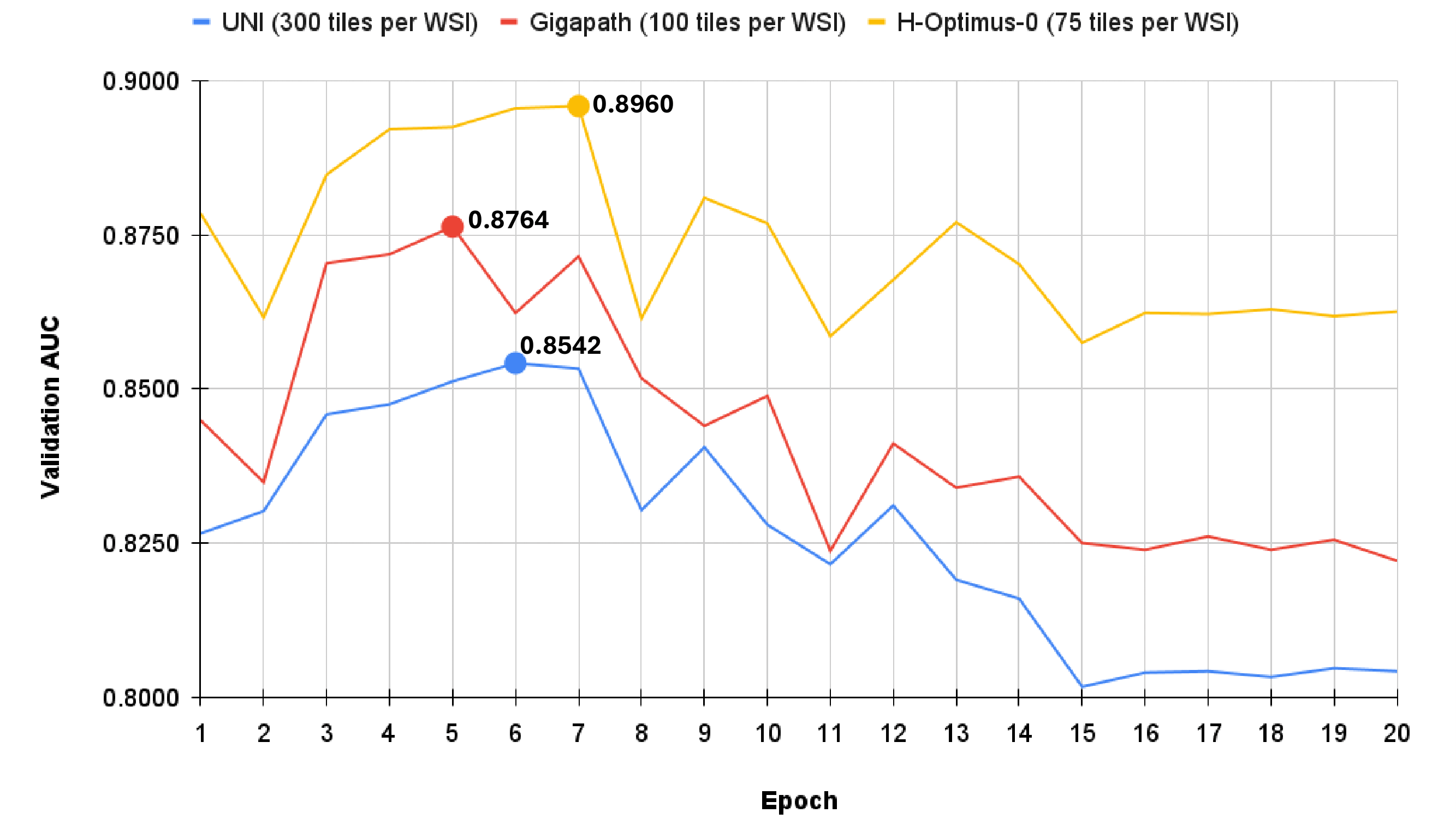}
        \caption{}
        \label{fig:fgfr3_blca_val}
    \end{subfigure}
    \hfill 
    \begin{subfigure}[b]{0.3\textwidth}
        \centering
        \includegraphics[width=\textwidth, keepaspectratio]{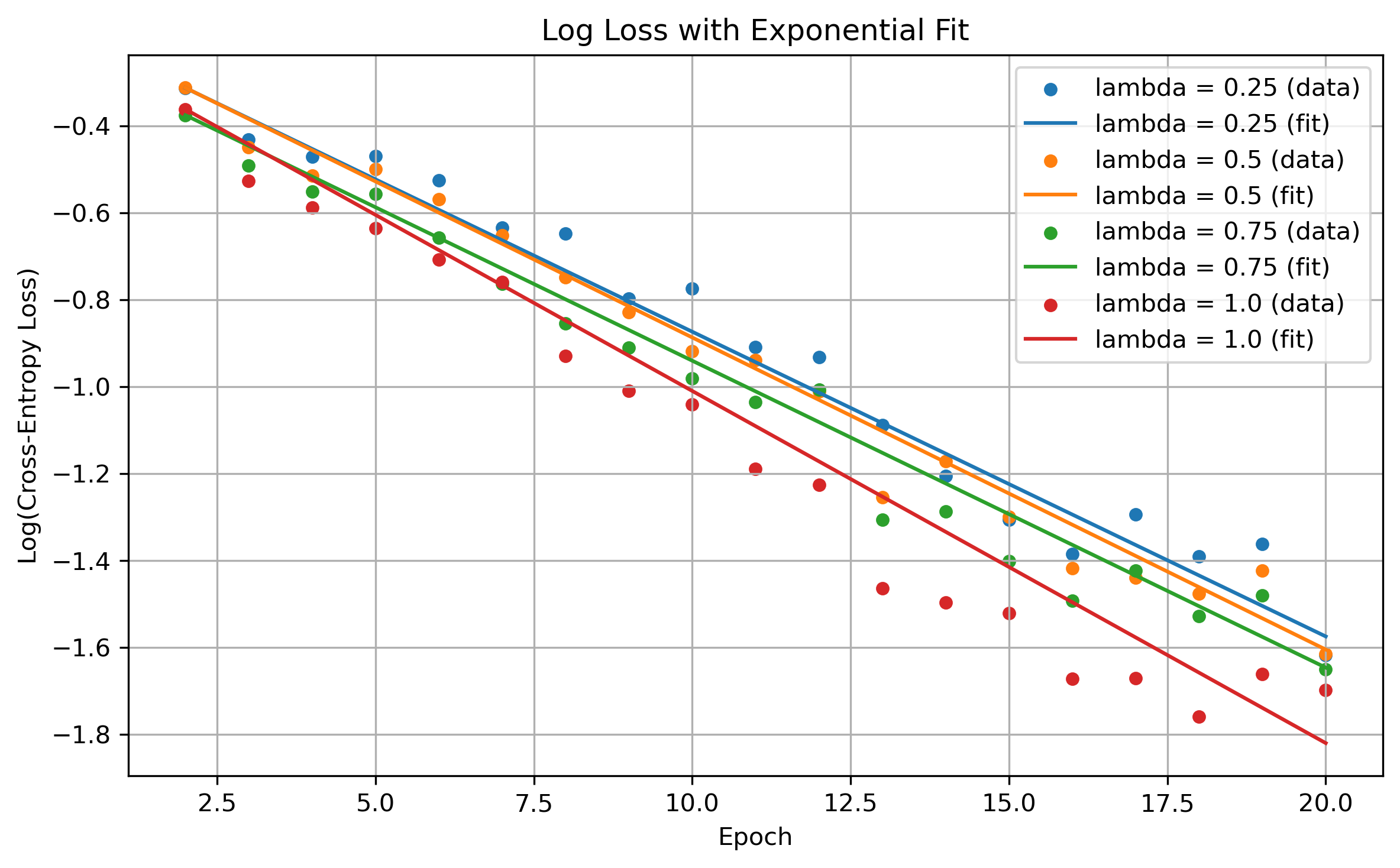}
        \caption{}
        \label{fig:lambda_variation}
    \end{subfigure}
    \hfill 
    \begin{subfigure}[b]{0.3\textwidth}
        \centering
        \includegraphics[width=\textwidth, keepaspectratio]{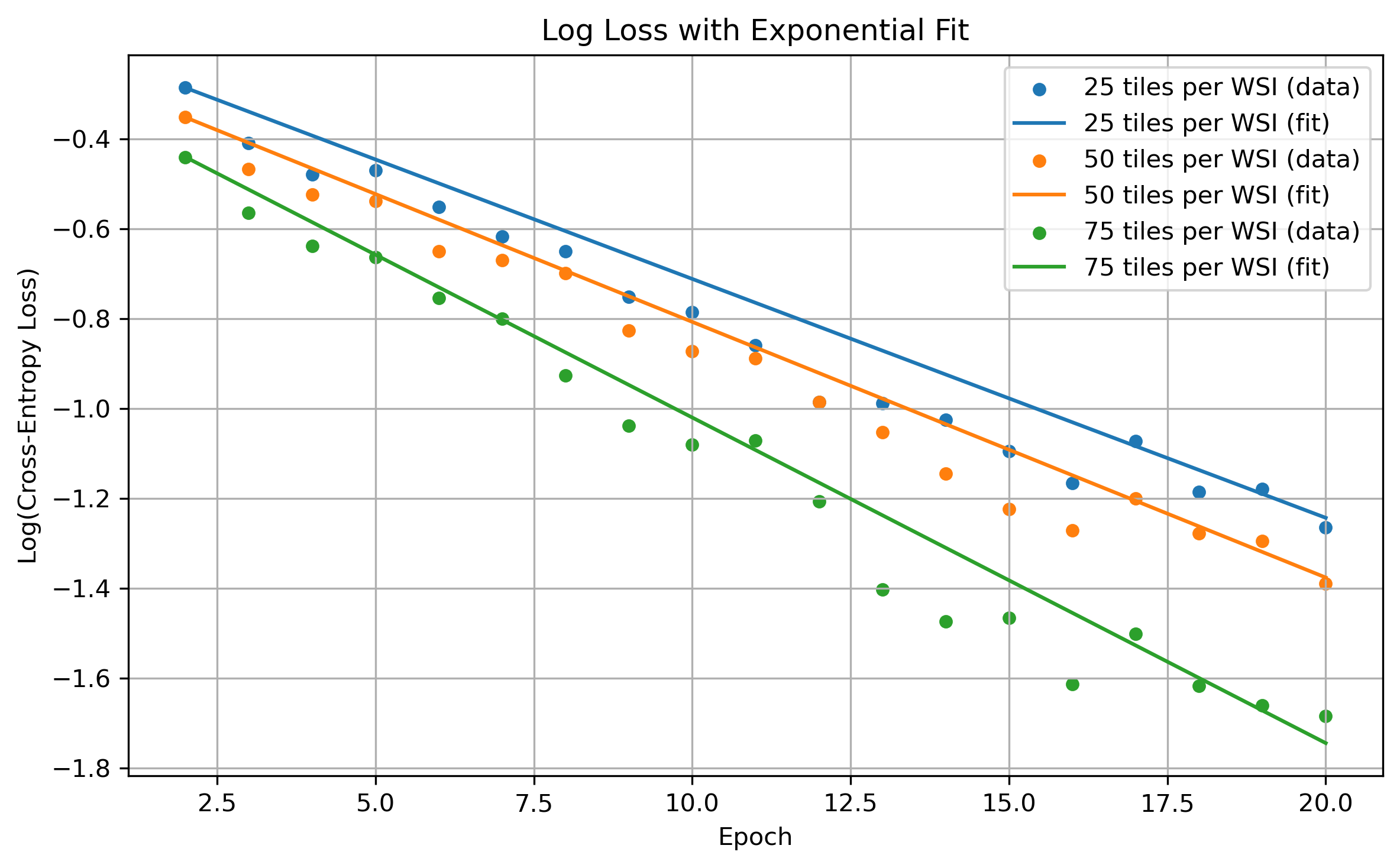}
        \caption{}
        \label{fig:training_loss}
    \end{subfigure}
    \caption{For bianry FGFR3 prediction task in BLCA -- (a) validation AUC trajectories of various TAPFM models, and the variation of log-transformed training loss of the H-Optimus-0 (TAPFM) model with (b) attention loss weighting ($\lambda$) and (c) tile sampling density.}
    \label{fig:ablation}
\end{figure}

\subsection{Ablation Studies}
\label{sec:ablation}
Systematic ablation studies investigate the influence of key hyperparameters -- attention loss weighting ($\lambda$ in equation~\ref{eq:pfmloss}) and number of tiles sampled per WSI per epoch -- on the performance of H-Optimus-0 (TAPFM) model for the binary FGFR3 prediction task in BLCA patients. 

\textbf{Lambda}: To evaluate the impact of the attention loss term in TAL (Equation~\ref{eq:pfmloss}), log-linear convergence fits of the form $\log \mathcal{L}_k = a + b\,k$ were computed over epochs $k = 2$ to $20$. Log-transformed loss fits are shown in Figure~\ref{fig:lambda_variation}, with raw curves in Appendix~\ref{fig:fgfr3_lambda_loss}. All values of $\lambda$ yielded consistent exponential decay, with convergence rates of $b = -0.0701$ ($R^2 = 0.9737$), $-0.0718$ ($R^2 = 0.9808$), $-0.0706$ ($R^2 = 0.9765$), and $-0.0810$ ($R^2 = 0.9693$) for $\lambda \in \{0.25, 0.5, 0.75, 1.0\}$, respectively. Although convergence behavior remained stable across this range, increasing $\lambda$ produced steeper decay and lower final training loss, with the best performance observed at $\lambda = 1.0$. These results suggest that stronger weighting of the TAL attention supervision improves training efficiency without compromising stability.

\textbf{Number of tiles}: Figure~\ref{fig:training_loss} shows that increasing the number of tiles per WSI consistently improved the performance of H-Optimus-0 (TAPFM) model (raw loss curves in Appendix~\ref{fig:fgfr3_ntiles_loss}). Another exponential decay model, $\log \mathcal{L}_k = a + b\,k$, was fit to training loss curves over epochs $k = 2$ to $20$, with estimated convergence rates of $b = -0.0532$ ($R^2 = 0.9731$), $-0.0569$ ($R^2 = 0.9811$), and $-0.0725$ ($R^2 = 0.9752$) for models trained with 25, 50, and 75 tiles per WSI, respectively. The 75-tile model converged by epoch 7 with a validation AUC of 0.8960, while the 50-tile model reached 0.8764 by the same epoch. The 25-tile variant (TAPFM H-Optimus-0) converged earlier—by epoch 5—but plateaued at a lower AUC of 0.8527. These findings indicate that denser tile sampling both accelerates convergence and improves performance on the institutional validation set.

Experimental analysis of incorporating a cosine regularization term (see Appendix~\ref{app:cosine_regularization}) for feature alignment loss in equation~\ref{eq:pfmloss} revealed no significant variation in FGFR3 prediction performance. Consequently, $\lambda = 1.0$ and the maximum number of tiles per WSI per epoch that could be accommodated on a single GPU for each PFM are used in all subsequent LUAD experiments

\section{Clinical Impact}
\label{sec:clinical_impact}
Because LUAD harbors a large number of drug‐targetable mutations, molecular diagnostics are especially valuable for guiding its targeted therapies~\cite{friedlaender_oncogenic_2024, chevallier_oncogenic_2021}. LUAD was the first tumor type to develop official guidelines recommending universal \emph{EGFR} and \emph{ALK} testing~\cite{lindeman_molecular_2013}, which were subsequently expanded in 2018 to include additional actionable genes when sufficient tissue is available~\cite{lindeman_updated_2018}. One of the principal challenges remains obtaining adequate DNA from small core biopsies to perform this extensive molecular workup~\cite{aisner_molecular_2012}. Although this paradigm has been embraced in LUAD, other malignancies such as BLCA have been sequenced far less frequently outside major centers despite the availability of FGFR3‐directed therapies~\cite{ascione_role_2023}. The ability to predict actionable mutations directly from H\&E-stained WSIs promises to streamline and economize current diagnostic workflows. By triaging or potentially obviating the need for costly and tissue-requiring molecular assays, TAPFM can help overcome the infrastructure and financial barriers that preclude precision oncology in resource-limited settings. Even when resources are available, TAPFM offers inference that can be performed as soon as the digital image is scanned. Therefore, clinicians can receive molecularly informed results within hours instead of days/weeks. This study demonstrates that a single model (TAPFM) can simultaneously detect multiple targetable mutations in LUAD, and that its performance generalizes from the originating institution to an independent TCGA cohort.

\section{Conclusion}
\label{sec:conclusion}
This work introduces TAPFM, a novel approach for adapting PFMs to specific clinical tasks by leveraging \vit{}'s attention mechanism for MIL aggregation and a detached dual-gradient approach for updating PFM parameters on a single GPU. TAPFM bridges the gap between self-supervised pretraining and supervised downstream adaptation in computational pathology, enabling more effective use of PFMs for clinical applications. The experimental results establish its effectiveness for clinically relevant mutation prediction tasks for BLCA (FGFR3) and LUAD (EGFR, KRAS, MET, ALK) patients while maintaining computational efficiency. Notably, TAPFM successfully tackles the challenging task of simultaneous prediction of four actionable mutations in LUAD patients, maintaining reasonable performance even for rare mutations like MET and ALK. Despite promising results, certain limitations of this work can be addressed in future studies. The approach shows potential for extension to additional clinical endpoints including survival analysis, recurrence prediction, and treatment response estimation. Investigations into which specific transformer layers benefit most from task adaptation could optimize the approach by selectively updating only those parameters. Additional external validation across multi-institutional cohorts with diverse scanning protocols and expanded biomarker panels would strengthen the clinical utility of the proposed approach. Implementations that scale to distributed training across multiple GPUs to increase the number of tiles processed per WSI during training may enhance TAPFM's generalization performance.




{\small
\bibliographystyle{unsrt}
\bibliography{references}
}

\appendix

\appendix
\section{Theoretical Analysis Details}

\subsection{Permutation Invariance}
\label{sec:perm_invariance}
A crucial property for MIL models is permutation invariance, which ensures that the order of instances in a bag does not affect the final prediction and below it is shown that the proposed TAPFM methods hold the permutation invariance property.

\begin{theorem}[Permutation Invariance]
Let $X = \{x_1, x_2, ..., x_K\}$ be a bag of instances and $\pi$ be a permutation of indices $\{1, 2, ..., K\}$. The bag representation $Z$ computed by TAPFM is invariant to permutation, i.e., $Z(X) = Z(\{x_{\pi(1)}, x_{\pi(2)},\dots, x_{\pi(K)}\})$.
\end{theorem}

\begin{proof}
Let $X = \{x_1, x_2, ..., x_K\}$ be a bag of instances and $\pi$ be a permutation of indices $\{1, 2, ..., K\}$. Let $X_\pi = \{x_{\pi(1)}, x_{\pi(2)}, ..., x_{\pi(K)}\}$ be the permuted bag.

Using matrix notation, let $\mathbf{Z} = [z_1^T, z_2^T, \ldots, z_K^T]^T \in \mathbb{R}^{K \times D}$ be the feature matrix and $\mathbf{a} = [a_1, a_2, \ldots, a_K]^T \in \mathbb{R}^K$ be the attention weight vector for the original ordering.

Let $\mathbf{P}_\pi \in \mathbb{R}^{K \times K}$ be the permutation matrix corresponding to $\pi$. Then, the feature matrix and attention vector for the permuted bag can be written as:
\begin{align}
\mathbf{Z}_\pi &= \mathbf{P}_\pi \mathbf{Z} \\
\mathbf{a}_\pi &= \mathbf{P}_\pi \mathbf{a}
\end{align}

The bag representation for the original ordering is:
\begin{equation}
Z(X) = \mathbf{Z}^T\mathbf{a} = \sum_{i=1}^{K} a_i z_i
\end{equation}

The bag representation for the permuted ordering is:
\begin{align}
Z(X_\pi) &= \mathbf{Z}_\pi^T\mathbf{a}_\pi \\
&= (\mathbf{P}_\pi \mathbf{Z})^T(\mathbf{P}_\pi \mathbf{a}) \\
&= \mathbf{Z}^T\mathbf{P}_\pi^T\mathbf{P}_\pi\mathbf{a} \\
&= \mathbf{Z}^T\mathbf{a} \\
&= Z(X)
\end{align}

The equality $\mathbf{P}_\pi^T\mathbf{P}_\pi = \mathbf{I}$ holds because permutation matrices are orthogonal. Therefore, $Z(X) = Z(X_\pi)$, which proves that the bag representation is permutation invariant.

In component form, this is equivalent to:
\begin{equation}
Z(X_\pi) = \sum_{i=1}^{K} a_{\pi(i)} z_{\pi(i)} = \sum_{j=1}^{K} a_j z_j = Z(X)
\end{equation}
where the change of variables $j = \pi^{-1}(i)$ is used.
\end{proof}






\subsection{Computational Complexity}
\label{app:comp_complexity}

\begin{theorem}[Computational Complexity]
The computational complexity of TAPFM for processing a single bag of $K$ instances is $\mathcal{O}(K \cdot C_{\text{ViT}})$, where $C_{\text{ViT}}$ is the complexity of a forward and backward pass through the specific ViT of a PFM for a single instance.
\end{theorem}

\begin{proof}
The computational complexity of TAPFM is analyzed by examining each step of the proposed Algorithm~\ref{alg:TAPFM} in detail:

\begin{itemize}
    \item \textbf{Feature Extraction (Forward Pass)}:
   
   Each of the $K$ tiles in the bag is passed through the ViT model to extract features. The complexity of processing a single tile depends on the ViT architecture:
   
   - Patch embedding: $\mathcal{O}(P^2 \cdot C \cdot D)$, where $P$ is the patch size, $C$ is the number of channels, and $D$ is the embedding dimension.
   - Self-attention layers: $\mathcal{O}(L \cdot N^2 \cdot D)$, where $L$ is the number of layers, $N$ is the number of tokens (patches).
   - MLP blocks: $\mathcal{O}(L \cdot N \cdot D^2)$.
   
   The total complexity for a single tile forward pass is $\mathcal{O}(C_{\text{ViT\_forward}}) = \mathcal{O}(P^2 \cdot C \cdot D + L \cdot N^2 \cdot D + L \cdot N \cdot D^2)$.
   
   For $K$ tiles, the total forward pass complexity is $\mathcal{O}(K \cdot C_{\text{ViT\_forward}})$.

\item \textbf{Attention Weight Computation}:
   
   For each tile $i$, the average attention weight is computed by aggregating over $H$ attention heads and $N$ tokens:
   \begin{equation}
   a_i = \frac{1}{H} \sum_{h=1}^{H} \frac{1}{N} \sum_{j=1}^{N} A^h_{cls,j}
   \end{equation}
   
   This requires $\mathcal{O}(H \cdot N)$ operations per tile, resulting in a total complexity of $\mathcal{O}(K \cdot H \cdot N)$ for all tiles.

\item \textbf{Bag Representation Computation}:
   
   Computing the weighted sum of feature vectors:
   \begin{equation}
   Z = \mathbf{Z}^T\mathbf{a} = \sum_{i=1}^{K} a_i z_i
   \end{equation}
   
   This requires $\mathcal{O}(K \cdot D)$ operations, where $D$ is the feature dimension.

\item \textbf{Aggregator Forward and Backward Pass}:
   
   The aggregator applies a linear transformation followed by a sigmoid activation:
   \begin{equation}
   \hat{y} = \sigma(W Z + b)
   \end{equation}
   
   The forward pass has complexity $\mathcal{O}(D)$ for the matrix-vector multiplication.
   
   Computing the aggregator loss and its gradient with respect to the bag representation has complexity $\mathcal{O}(D)$.
   
   Updating the aggregator parameters has complexity $\mathcal{O}(D)$.

\item \textbf{Computing Gradients for the PFM}:
   
   The feature loss is:
   \begin{equation}
   \mathcal{L}_{feature} = -\text{tr}(\mathbf{Z}\mathbf{G}_z^T) = \sum_{i=1}^{K} \langle z_i, g_{z_i} \rangle
   \end{equation}
   
   Computing this loss has complexity $\mathcal{O}(K \cdot D)$.
   
   The attention loss is:
   \begin{equation}
   \mathcal{L}_{attention} = \mathbf{a}^T\mathbf{g}_a = \sum_{i=1}^{K} a_i \cdot g_{a_i}
   \end{equation}
   
   Computing this loss has complexity $\mathcal{O}(K)$.

\item \textbf{PFM Backward Pass}:
   
   The backward pass through the ViT for all $K$ tiles has a complexity of $\mathcal{O}(K \cdot C_{\text{ViT\_backward}})$, which is typically on the same order as the forward pass:
   $\mathcal{O}(C_{\text{ViT\_backward}}) = \mathcal{O}(P^2 \cdot C \cdot D + L \cdot N^2 \cdot D + L \cdot N \cdot D^2)$.
\end{itemize}
Combining all steps, the total computational complexity is:
\begin{equation}
\mathcal{O}(K \cdot C_{\text{ViT\_forward}} + K \cdot H \cdot N + K \cdot D + D + K \cdot D + K + K \cdot C_{\text{ViT\_backward}})
\end{equation}

Since $C_{\text{ViT}} = C_{\text{ViT\_forward}} + C_{\text{ViT\_backward}}$ and typically $C_{\text{ViT}} \gg H \cdot N$ and $C_{\text{ViT}} \gg D$, the overall complexity is dominated by the ViT forward and backward passes, resulting in $\mathcal{O}(K \cdot C_{\text{ViT}})$.

This complexity analysis shows that the proposed TAPFM approach scales linearly with the number of instances in the bag. The constant factor $C_{\text{ViT}}$ depends on the specific architecture of the Vision Transformer but is independent of the bag size.
\end{proof}

\subsection{Space Complexity Analysis}
\label{app:spacecomplexity}
In addition to time complexity, space complexity is a critical consideration for practical deployment of TAPFM, especially for gigapixel WSIs processed using transformer-based models. 

\begin{theorem}[Space Complexity]
The space complexity of TAPFM for processing a single bag of $K$ instances (tiles) is $\mathcal{O}(|\theta_{PFM}| + K \cdot (S_{act} + S_{grad}))$, where $|\theta_{PFM}|$ is the size of the PFM parameters, and $S_{act}$ and $S_{grad}$ are the memory requirements for activations and gradients per instance, respectively.
\end{theorem}

\begin{proof}
The space complexity of TAPFM is analyzed by examining the memory requirements of each component in the proposed approach as follows:

\begin{itemize}
    \item \textbf{Model Parameter Storage}: The memory required to store the \vit parameters is $\mathcal{O}(|\theta_{PFM}|)$, which can be further decomposed as $\mathcal{O}(L \cdot N \cdot D^2)$, where $L$ is the number of transformer layers, $N$ is the number of tokens, and $D$ is the embedding dimension. 
    
    
    \item \textbf{Feature Matrix Storage}: For a bag of $K$ instances, the feature matrix $\mathbf{Z} \in \mathbb{R}^{K \times D}$ requires $\mathcal{O}(K \cdot D)$ memory.
    
    \item \textbf{Attention Vector Storage}: The attention weight vector $\mathbf{a} \in \mathbb{R}^K$ requires $\mathcal{O}(K)$ memory.
    
    \item \textbf{Activation Memory}: During the forward pass, each instance requires storing intermediate activations for all transformer layers. For a single instance, this requires $\mathcal{O}(L \cdot N \cdot D)$ memory for hidden states and $\mathcal{O}(L \cdot H \cdot N^2)$ for attention maps, where $H$ is the number of attention heads. Let's denote the total activation memory per instance as $S_{act} = \mathcal{O}(L \cdot N \cdot D + L \cdot H \cdot N^2)$.
    
    \item \textbf{Gradient Memory}: During backpropagation, gradients for both activations and parameters must be stored. The memory requirement for gradients per instance is denoted as $S_{grad}$, which is typically on the same order as $S_{act}$.
    
    \item \textbf{Aggregator Memory}: The aggregator requires $\mathcal{O}(D)$ memory for parameters and $\mathcal{O}(1)$ for the output, which is negligible compared to the PFM memory requirements.
    
    \item \textbf{Gradient Matrices for Feature and Attention}: The feature gradient matrix $\mathbf{G}_z \in \mathbb{R}^{K \times D}$ requires $\mathcal{O}(K \cdot D)$ memory, and the attention gradient vector $\mathbf{g}_a \in \mathbb{R}^K$ requires $\mathcal{O}(K)$ memory.
\end{itemize}

Combining all components and identifying the dominant terms, the total space complexity is:
\begin{equation}
\mathcal{O}(|\theta_{PFM}| + K \cdot D + K + K \cdot S_{act} + K \cdot S_{grad} + K \cdot D + K)
\end{equation}

Since $S_{act}$ and $S_{grad}$ are typically much larger than $D$, and combining like terms produces:
\begin{equation}
\mathcal{O}(|\theta_{PFM}| + K \cdot (S_{act} + S_{grad}))
\end{equation}

Above equation indicates that the memory requirements of TAPFM scale linearly with the number of instances (tiles) $K$. However, the constant factors $S_{act}$ and $S_{grad}$ can be substantial for large \vit models, potentially limiting the number of instances that can be processed simultaneously on a single GPU.
\end{proof}


    
    
    


\subsection{Task-Specific Adaptation without Catastrophic Forgetting}
\label{app:catastrophic_forgetting}

\begin{proposition}[Preservation of Pretrained Knowledge]
The proposed TAPFM approach effectively implements a specialized form of continual learning that prevents catastrophic forgetting. The Fisher information matrix for pretrained parameters is implicitly preserved:
\begin{equation}
F_{\theta_{PFM}} = \mathbb{E}_{p(x,y)}[\nabla_{\theta_{PFM}} \log p(z|x;\theta_{PFM}) \nabla_{\theta_{PFM}} \log p(z|x;\theta_{PFM})^T]
\end{equation}
\end{proposition}

\begin{proof}
Catastrophic forgetting occurs when fine-tuning a pre-trained model on a new task destroys the knowledge acquired during pretraining. In information-theoretic terms, this happens when parameter updates move away from regions of high Fisher information with respect to the pretraining task.

The Fisher information matrix $F_{\theta_{PFM}}$ characterizes the curvature of the loss landscape around the pretrained parameters. It quantifies how sensitive the model's output is to small changes in each parameter. Parameters with high Fisher information are crucial for the model's performance on the pretraining task.

In standard fine-tuning, the gradient update is:
\begin{equation}
\Delta \theta_{PFM} = -\eta \nabla_{\theta_{PFM}} \mathcal{L}_{task}
\end{equation}

This update does not account for the importance of parameters for the pretraining task, potentially leading to catastrophic forgetting. In contrast, optimal updates for preserving pretrained knowledge while adapting to a new task would be of the form:
\begin{equation}
\Delta \theta_{PFM} = -\eta F_{\theta_{PFM}}^{-1} \nabla_{\theta_{PFM}} \mathcal{L}_{task}
\end{equation}

This is a form of natural gradient descent, which adapts the update direction based on the curvature of the loss landscape.

The proposed TAPFM approach implicitly approximates this behavior through gradient detachment. By separating the aggregator and PFM optimization, TAPFM allows the PFM to adapt more conservatively. The detached gradients from the aggregator act as a filtered signal that guides the PFM to adapt while respecting its pretrained structure.

Specifically, the proposed PFM update is:
\begin{equation}
\Delta \theta_{PFM} = -\eta \nabla_{\theta_{PFM}} \mathcal{L}_{PFM}
\end{equation}

where $\mathcal{L}_{PFM}$ is defined using detached gradients from the aggregator. This approach resembles a regularized optimization problem:
\begin{equation}
\min_{\theta_{PFM}} \mathcal{L}_{task}(\theta_{PFM}) + \lambda \cdot \mathcal{R}(\theta_{PFM} - \theta_{PFM}^{pretrained})
\end{equation}

where $\mathcal{R}$ is an implicit regularization term that penalizes deviations from the pretrained parameters. The solution to this regularized problem can be approximated as:
\begin{equation}
\Delta \theta_{PFM} \approx -\eta (I + \lambda \cdot F_{\theta_{PFM}})^{-1} \nabla_{\theta_{PFM}} \mathcal{L}_{task}
\end{equation}

As $\lambda$ increases, this approaches the natural gradient update:
\begin{equation}
\Delta \theta_{PFM} \approx -\eta \lambda^{-1} F_{\theta_{PFM}}^{-1} \nabla_{\theta_{PFM}} \mathcal{L}_{task} \approx -\eta' F_{\theta_{PFM}}^{-1} \nabla_{\theta_{PFM}} \mathcal{L}_{task}
\end{equation}

where $\eta' = \eta \lambda^{-1}$ is an effective learning rate.

The detached gradient approach of TAPFM serves as an implicit regularization mechanism that enables task-specific adaptation while preserving the foundational knowledge embedded in the pretrained model. This balance between adaptation and preservation is especially important for PFMs, where pretraining captures general tissue morphology and fine-tuning adapts to specific clinically relevant tasks.
\end{proof}

\subsection{Cosine Regularization for Feature Alignment Loss}
\label{app:cosine_regularization}
The feature alignment loss in equation~\ref{eq:featureloss} encourages feature vectors to move in the direction that reduces the classification loss. However, in certain scenarios, direct optimization of this loss may lead to updates where feature vectors and gradients point in opposing directions, potentially causing unstable training dynamics. To mitigate this issue, a cosine regularization term is considered. The cosine regularization term is defined as: $\mathcal{L}_{reg} = \sum_{i=1}^{K} (1 - \cos(z_i, g_{z_i}))$
where $\cos(z_i, g_{z_i})$
represents the cosine similarity between the feature vector $z_i$
and its gradient $g_{z_i}$, calculated as:
\begin{equation}
\cos(z_i, g_{z_i}) = \frac{\langle z_i, g_{z_i} \rangle}{\|z_i\|_2 \cdot \|g_{z_i}\|_2}
\end{equation}

This regularization term penalizes scenarios where feature vectors and their gradients are oriented in opposite directions (negative cosine similarity). When incorporated into the task adaptation loss, the complete objective becomes:
\begin{equation}
\mathcal{L}_{PFM} = \mathcal{L}_{feature} + \lambda \mathcal{L}_{attention} + \beta \mathcal{L}_{reg}
\end{equation}
where $\beta$ is a hyperparameter controlling the strength of the regularization. The purpose of this regularization is to promote more stable optimization by encouraging gradual updates to the learned representations while still allowing adaptation to the downstream task. Theoretically, this prevents drastic changes in feature space that might compromise the pretrained knowledge while ensuring that updates remain aligned with the classification objective.


\section{Training Loss Curves}

\begin{figure}[!htb]
    \centering
    
    \begin{subfigure}[b]{0.45\textwidth}
        \centering
        \includegraphics[width=\textwidth, keepaspectratio]{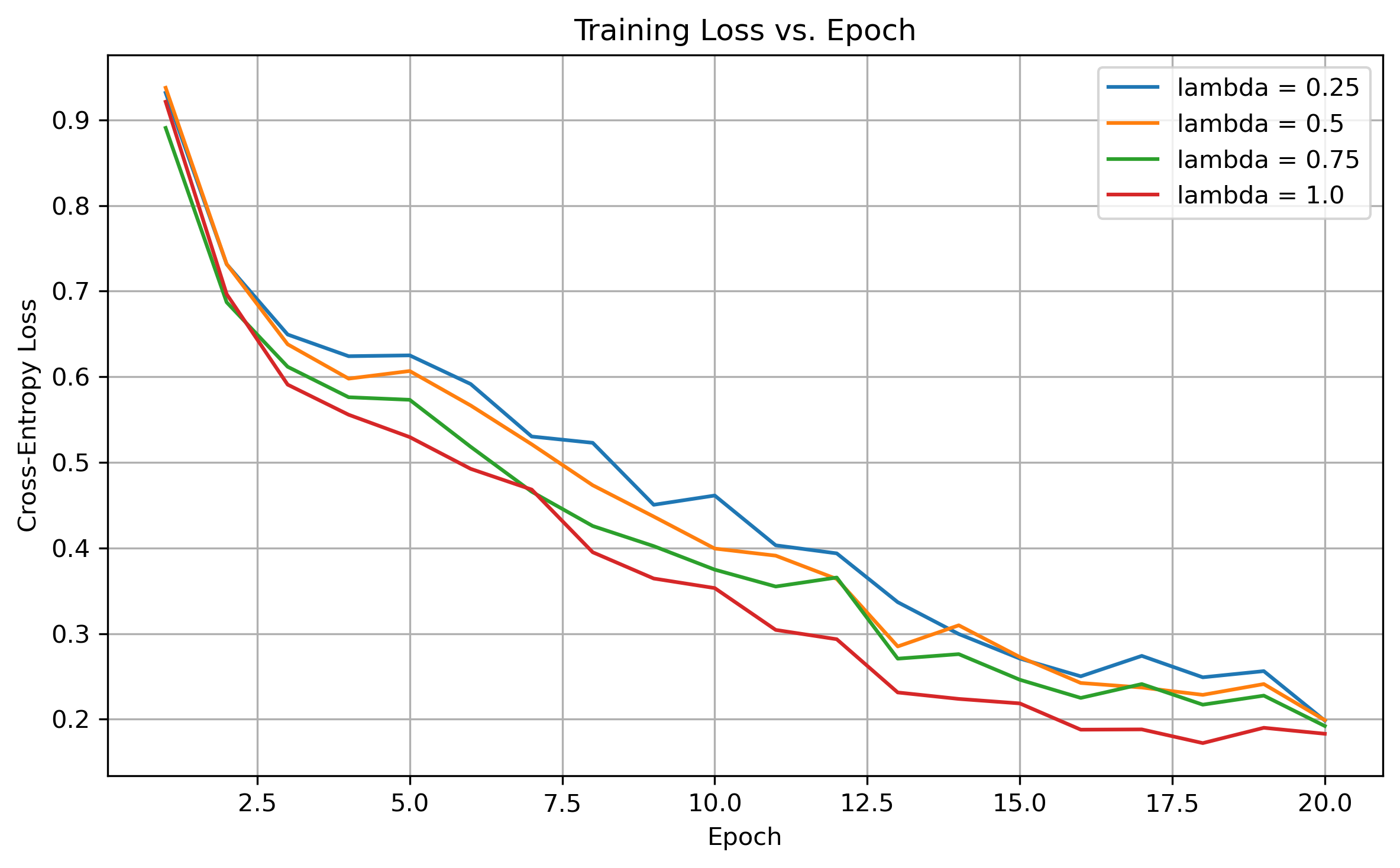}
        \caption{Training loss variation with respect to lambda in TAL (equation~\ref{eq:pfmloss}) for H-Optimus-0 (TAPFM) model.}
        \label{fig:fgfr3_lambda_loss}
    \end{subfigure}
    \hfill 
    \begin{subfigure}[b]{0.45\textwidth}
        \centering
        \includegraphics[width=\textwidth, keepaspectratio]{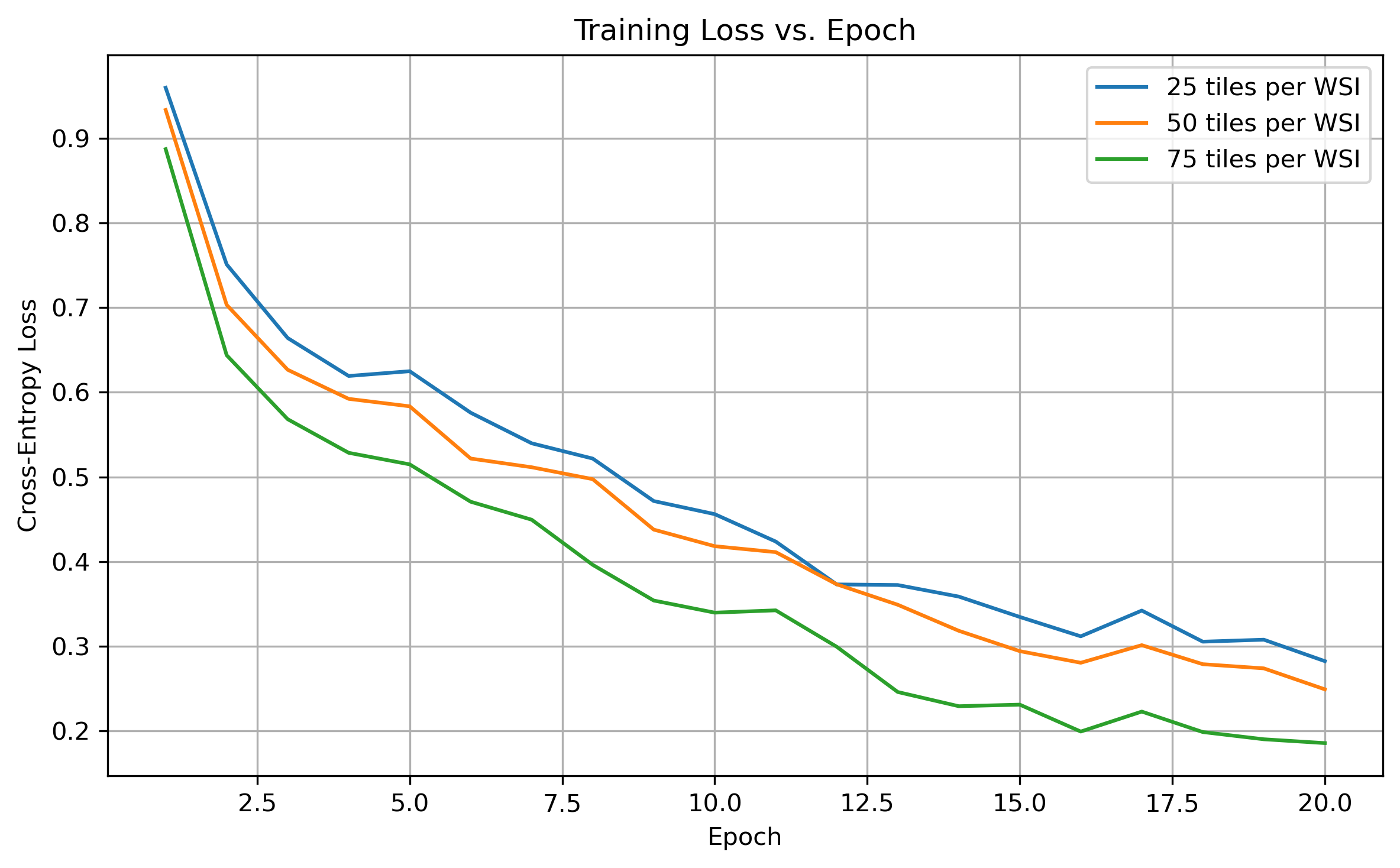}
        \caption{Training loss variation with the number of tiles for H-Optimus-0 (TAPFM) model.}
        \label{fig:fgfr3_ntiles_loss}
    \end{subfigure}
    \caption{Ablations studies for FGFR3 prediction in BLCA patients.}
    \label{fig:ablation_raw}
\end{figure}

\end{document}